\documentclass{article}
\usepackage[margin=1in]{geometry} 
\usepackage{xcolor}

\usepackage{natbib}
\PassOptionsToPackage{authoryear, round}{natbib}

\usepackage{amsthm} 

\usepackage[utf8]{inputenc} 
\usepackage[T1]{fontenc}    
\usepackage{hyperref}       
\usepackage{url}            
\usepackage{booktabs}       
\usepackage{amsfonts}       
\usepackage{nicefrac}       
\usepackage{microtype}      
\usepackage{xcolor}         
\usepackage{graphicx}
\usepackage{bm}
\usepackage{algorithm}
\usepackage{algorithmic}
\usepackage{bbm}

\usepackage{amsmath}
\usepackage{amsthm}

\usepackage{amssymb}
\usepackage{siunitx}

\DeclareMathOperator*{\argmax}{arg\,max}

\newtheorem*{repthm}{Theorem \ref{thm:crc_utility_opt}}

\newcommand{\R}{{\mathbb{R}}}
\newcommand{\E}{{\mathbb{E}}}

\newtheorem{theorem}{Theorem}

\newtheorem{definition}{Definition}

\title{Conformal Arbitrage: Risk-Controlled Balancing of Competing Objectives in Language Models
}

\author{%
  William Overman \\
  Graduate School of Business \\
  Stanford University \\
  \texttt{wpo@stanford.edu}
  \and
  Mohsen Bayati \\
  Graduate School of Business \\
  Stanford University \\
  \texttt{bayati@stanford.edu}
}

\begin{document}

\maketitle

\begin{abstract}
Modern language‑model deployments must often balance \emph{competing objectives}—for example, helpfulness versus harmlessness, cost versus accuracy, and reward versus safety.
We introduce \textbf{Conformal Arbitrage}, a post‑hoc framework that learns a data‑driven threshold to mediate between a \emph{Primary} model optimized for a primary objective and a more conservative \emph{Guardian}—which could be another model or a human domain expert—aligned with a guardrail objective. 
The threshold is calibrated with conformal risk control, yielding finite‑sample, distribution‑free guarantees that the long‑run frequency of undesirable events (such as factual errors or safety violations) does not exceed a user‑specified quota.  
Because Conformal Arbitrage operates wholly at the API level—without requiring access to model logits or updating model weights—it complements weight‑based alignment techniques and integrates seamlessly with existing cost‑aware cascades.  
Empirically, Conformal Arbitrage traces an efficient frontier, allowing users to define an acceptable performance level for one objective while maximizing utility in another. We observe that our method outperforms (in terms of accuracy) cost-matched random routing between models. These properties make Conformal Arbitrage a practical, theoretically grounded tool for trustworthy and economical deployment of large language models across a broad range of potentially competing objectives.
\end{abstract}

\section{Introduction}

Large language models (LLMs) excel at reasoning, coding, and open‑domain question answering, yet real‑world deployments frequently need to navigate tensions between potentially competing objectives such as \emph{helpfulness} and \emph{harmlessness} or \emph{cost} and \emph{accuracy}.

Current practices mostly tackle the tension between helpfulness and harmlessness by \emph{modifying the model itself}: reinforcement learning from human feedback (RLHF) \citep{christiano2017deep, ouyang2022training}, direct–preference optimisation (DPO) \citep{rafailov2023dpo}, Constitutional~AI \citep{bai2022constitutional}, or multi‑objective fine‑tuning \citep{zhou2023beyond, wang2024arithmetic} each produce a \emph{single} operating point along the Pareto frontier.
While powerful, these methods demand expensive data collection, GPU‑intensive retraining, and --- for API‑only models --- are often not applicable.

For the cost versus accuracy tradeoff, there has been significant work on cascades: a cheap model handles easy queries and defers the rest to a stronger fallback \citep{chen2023frugalgpt, aggarwal2025automix, zellinger2025earlyabstention}.
Recently, \cite{jung2025trust} introduced Cascaded Selective Evaluation (CSE), calibrating per-model confidence estimators via fixed-sequence multiple testing to obtain rigorous guarantees on alignment to human pairwise preferences. However, these approaches are tailored for controlling a binary disagreement risk, while a user may be interested in controlling arbitrary guardrail metrics at deployment time.

We introduce \textbf{Conformal Arbitrage (CA)}, a lightweight router that sits \emph{outside} the language models. The term ``arbitrage'' captures how our approach exploits the performance gap between specialized models to achieve superior outcomes than naive selection between models. Given \textbf{(i)} a \emph{Primary} model optimized for the primary objective and \textbf{(ii)} a more conservative \emph{Guardian} model or a human domain expert, aligned with a guardrail objective, CA offers a principled alternative to randomized routing between models. Instead of merely alternating between models with some probability, CA learns a single scalar threshold on how strongly the Primary model favors its top choice over alternatives (a notion we formally define as ``score gap'' later in the paper). This threshold determines when the Primary model's confidence is sufficient to act upon its prediction versus when to defer to the Guardian, creating a principled decision boundary that optimizes the trade-off between objectives.

The threshold is calibrated using \emph{conformal risk control} (CRC) \citep{angelopoulos2024conformal}, yielding \emph{finite‑sample, distribution‑free guarantees} that the long‑run frequency (or magnitude) of undesirable events never exceeds a user‑specified budget~$\alpha$. This enables precise control over trade-offs---users can explicitly specify how much they are willing to compromise on one objective to gain on the other. Because CA touches \emph{no model weights}, it complements weight‑based alignment and applies to closed, black‑box APIs, making it a remarkably lightweight approach to achieving Pareto improvements over simple model selection strategies.

Our experiments examine both the the cost versus accuracy tradeoff using the the TruthfulQA and MMLU benchmarks as well as the helpfulness versus harmlessness tradeoff using the PKU-SafeRLHF benchmark. Across all settings CA traces an efficient frontier that outperforms random or cost‑matched routing baselines.

Conformal Arbitrage transforms an immutable, potentially unpredictable LLM (or a family of LLMs) into a controllable system whose risk–utility position can be \emph{dialed after deployment}. In our experiments, we demonstrate this capability using state-of-the-art LLMs from the GPT-4.1 series, \cite{openai2025gpt41}, showing how our method enables fine-grained control over various tradeoffs without modifying the underlying models. By requiring only a few hundred logged examples for calibration, CA offers a pragmatic path toward trustworthy, cost‑efficient and customizable language‑model services that can be adjusted to meet evolving requirements long after initial deployment.

\section{Related work}

Real--world deployments must strike a pragmatic balance between \emph{helpfulness}---supplying users with accurate and detailed information---and \emph{harmlessness}---avoiding policy‐breaking or dangerous content.  Early alignment work framed the problem as a single--objective optimization: RLHF \citep{christiano2017deep,ouyang2022training} and its variant DPO \citep{rafailov2023dpo} collapse nuanced feedback into a \emph{single} reward model and therefore deliver one operating point on the Pareto frontier.  Subsequent methods introduced explicit two–factor training: RLHF on mixed helpful–harmless datasets \citep{bai2022training}, Constitutional~AI’s self‐revision loop \citep{bai2022constitutional}, and Bi‑Factorial Preference Optimisation (BFPO) \citep{zhang2025bifactorial} that casts the bi‑objective RLHF loss as a direct supervised criterion.  Safe‑RLHF \citep{safe_rlhf2023} separates a reward and a cost head and enforces constraints by Lagrangian relaxation, while Circuit Breakers intervene at generation time to halt policy‑violating continuations \citep{zou2024circuitbreaker}.

The PKU-SafeRLHF benchmark \citep{ji2023beavertails} was specifically introduced to quantify this helpfulness-harmlessness trade-off, providing dual annotations that enable researchers to measure progress on both dimensions simultaneously. Anthropic's Constitutional AI \citep{bai2022constitutional} further explores alignment by embedding principles directly into model training. More recently, the MoGU framework \citep{du2024mogu} dynamically routes between model variants optimized separately for usability and safety.
Empirically, while these approaches curb unsafe completions, they still lock the model into one fixed balance point between helpfulness and harmlessness.

Beyond helpfulness–harmlessness many other objectives--- accuracy, cost, latency, fairness, demographic parity, domain–specific risk, etc.---can be in conflict. Many recent works have proposed weight–based strategies to navigate the resulting frontiers between such competing objectives.  Rewarded soups linearly interpolates checkpoints fine‑tuned on distinct rewards to trace that surface \citep{rame2023rewarded}, Directional Preference Alignment adds multiple reward heads for steerable inference \citep{wang2024arithmetic}, MaxMin‑RLHF learns a mixture of reward models to protect minority preferences \citep{chakraborty2024maxmin}, and MO‑DPO converts several preference signals into a closed‑form multi‑objective loss \citep{zhou2023beyond}.  These approaches nevertheless share two limitations: \textbf{(i)} they require access to model weights and retraining, and \textbf{(ii)} they provide no theoretical guarantees that the inherent guardrail metrics driving the trade‑off (e.g., safety, accuracy, or cost) will stay within a user‑specified budget.

In contrast, our method of Conformal Arbitrage is weight‑agnostic and sits \emph{outside} the LLM.  By calibrating a single threshold with conformal risk control \citep{angelopoulos2024conformal}, it transforms any pair of black‑box models, one of which can be a human, into a \emph{continuum} of operating points with \emph{provable} finite‑sample bounds on the chosen guardrail metric (e.g. harmlessness). 

Conformal Arbitrage is thus closely tied to routing and cascade approaches that tackle cost--accuracy trade‑offs \citep{chen2023frugalgpt,yue2024large,ong2024routellm,aggarwal2025automix,zellinger2025earlyabstention,varangot2025doing}, but can be used to tackle any potential pair of objectives that may be in tension, thus abstractly covering cost–accuracy cascades as a special case. 

However, unlike these previous approaches we make no particular optimizations for any specific trade-off, including cost and accuracy, and we do not claim to out-perform such cascade systems on metrics for which they are explicitly optimized. Furthermore, compared to most routing approaches that rely on complex learned functions to distribute queries between models \citep{varangot2025doing}, Conformal Arbitrage employs a principled, theoretically-grounded method using a single calibrated scalar threshold. 


Scalable-oversight research explores how weaker agents or humans can be organized into critique hierarchies that amplify limited supervision.
Amplification and Debate delegate verification to inexpensive judges and, under certain complexity assumptions, achieve provable “weak-to-strong” guarantees \citep{christiano2018amplification,irving2018debate, burns2023weak}.
Process supervision instead labels intermediate reasoning steps so that mistakes are caught early \citep{lightman2023process}.
Self-reflection frameworks ask a model to generate critiques (and often revisions) of its own outputs \citep{madaan2023selfrefine,yang2024deepcritic,tang2024scrit}. 
Post-hoc risk control strategies in model deployment have also gained attention, particularly through moderation and oversight models deployed by industry leaders \citep{openai2023gpt4systemcard}.
Conformal Arbitrage complements these approaches with a statistically sound escalation rule: the Primary acts autonomously within a risk budget, else forwards a slimmed-down slate of actions to a human or Guardian. Finite-sample bounds from conformal risk control budget both the Guardian’s load and the residual risk, giving a lightweight post-hoc route to scalable oversight without retraining.

The underlying selective routing approach of our work resonates with classical selective prediction and reject-option frameworks initially formalized by \cite{chow1970optimum} and later refined in modern selective classification research \citep{geifman2019selectivenet}.

 Conformal prediction (CP) and its generalization, conformal risk control (CRC) \citep{vovk,bates2021distributionfree,angelopoulos2024conformal}, provide distribution-free, finite-sample guarantees that make them generally attractive post-hoc alignment tools for high-stakes LLM deployments. For instance, \cite{chen2025conformaltailriskcontrol} align language models with human risk judgments by controlling tail risks such as toxicity, while \cite{su2024apienoughconformalprediction} demonstrate conformal prediction applied effectively to black-box LLM APIs without internal access. Additionally, conformal risk control has been leveraged in deployment scenarios such as action deferral, illustrated by the KnowNo framework \citep{ren2023robots}, which uses conformal uncertainty quantification to trigger human oversight.
  
 Conformal prediction and conformal risk control have been used to filter low-confidence QA answers \citep{kumar2023conformal}, retain only entailment-supported sub-claims \citep{mohri2024language}, and bound hallucination rates via abstention \citep{abbasi2024mitigating}. Beyond marginal guarantees, conditional and adaptive CRC tighten coverage on hard prompts \citep{cherian2024large}, and sampling-based set prediction extends CP to free-text generation \citep{quach2024conformal}. Framing alignment as property testing, \citet{overman2024aligning} calibrate outputs to satisfy safety or fairness constraints without retraining. Building on this lineage, we adapt CRC to learn a risk-calibrated switch between a Primary model and a Guardian model without retraining either model.

Conformal Arbitrage is most closely related to \emph{Cascaded Selective Evaluation} (CSE) of \citet{jung2025trust}.  CSE equips each judge with a confidence score, calibrates a per-judge threshold, and escalates through a cascade until some judge is confident, thereby controlling the Bernoulli risk that a machine-preferred answer disagrees with human majority.  Conformal Arbitrage addresses more general tradeoffs: it controls \emph{any} bounded guardrail loss (safety, accuracy, cost, latency, etc.) and can filter a large action space to a smaller candidate set that a Guardian or human refines, rather than abstaining on the whole instance. CSE’s \textit{Simulated Annotators} requires $K$-shot prompting (for $K$ examples of preference annotations) the model $N$ different times (for $N$ human annotators) in order to obtain an ensemble prediction \emph{and} access to predictive probabilities extracted from the model's \texttt{logprobs}, so every judge call is multiplied many-fold and is limited to APIs that expose token-level logits.  Conformal Arbitrage, by contrast, needs at most \emph{one} call to the Primary and (when routed) \emph{one} to the Guardian, treats the returned scores as opaque, requiring no access to logits or probabilities, and thus works with strictly black-box APIs.


\section{Preliminaries}
\label{sec:preliminaries}

Conformal Arbitrage uses \emph{conformal risk control} (CRC) to supply finite‑sample, distribution‑free guarantees on the guardrail metric while treating the underlying language models as black boxes.  CRC extends the framework
of \emph{conformal prediction} (CP)
\citep{vovk,bates2021distributionfree} from binary
error control to control of \emph{arbitrary bounded risks}.  We briefly
summarize both ideas.

\paragraph{Conformal prediction}
Let $\mathcal{X}$ and $\mathcal{Y}$ be the input and output spaces, equipped with a joint probability distribution, and draw an exchangeable sample $(X_i,Y_i)_{i=1}^{n+1}\!\sim\!P$ where the first $n$ sample are used for calibration, and $(X_{n+1},Y_{n+1})$ is used for testing.  
Given any predictor $f:\mathcal{X}\!\to\!\mathcal{Y}$ and score $s_f(x,y)$ (e.g.\ $|y-f(x)|$), let $q_{1-\alpha}$ be the $(1-\alpha)$ empirical quantile of $\{s_f(X_i,Y_i)\}_{i=1}^{n}$.  
The conformal set is defined by
$C(x)=\{\,y\in\mathcal{Y}\!:\,s_f(x,y)\le q_{1-\alpha}\}$, and
enjoys the finite-sample guarantee
$\Pr\{\,Y_{n+1}\!\notin\!C(X_{n+1})\}\le\alpha$.
Thus any black-box predictor attains $(1-\alpha)$ coverage without distributional assumptions \citep{vovk,bates2021distributionfree}.

\paragraph{Conformal risk control}
Many real‑world objectives are not binary mistakes but expectations of
a task‑specific loss—for example, safety‑violation rate, factual errors, mean latency,
or excess dollar cost.  Conformal risk control \citep{angelopoulos2024conformal} handles
such objectives by introducing a \emph{bounded, non‑increasing} loss
curve $L_i(\lambda)\in[0,B]$, where $B$ is an upper bound on the loss, for each calibration point, indexed by a
tunable threshold $\lambda\in\Lambda \subset\mathbb{R}$.  Defining the empirical risk
$\hat R_n(\lambda)=\tfrac1n\sum_{i=1}^{n}L_i(\lambda)$, CRC selects
\begin{equation}
  \hat{\lambda}
  \;=\;
  \inf\Bigl\{
       \lambda\in\Lambda:
       \frac{n}{n+1}\,\hat R_n(\lambda)+\frac{B}{n+1}\le\alpha
     \Bigr\},
  \label{eq:crc}
\end{equation}
and proves the \emph{finite‑sample} guarantee for 
\(
\mathbb{E}\bigl[L_{n+1}(\hat\lambda)\bigr]\le\alpha,
\)
again under assumption of exchangeability between the calibration data and test point.  Choosing
$L_i(\lambda)=\mathbb{I}\{Y_i\notin C_\lambda(X_i)\}$ recovers
classical CP; alternative losses yield risk bounds tailored to
deployment needs.


\section{Methodology: conformal arbitrage}

We aim to invoke a \emph{Primary} model as often as possible
(e.g.\ a helpfulness-maximizing or low-cost model) while ensuring, with high
confidence, that a critical requirement (e.g.\ harmlessness, accuracy) is
satisfied by routing calls to a \emph{Guardian} model (or human) as needed. The linkage between the two models is formalized through conformal risk control \citep{angelopoulos2024conformal}.

\subsection{Setting}

Let \(\{x_i\}_{i\ge 1}\) be an exchangeable sequence of
\(\mathcal{X}\)-valued random variables that we refer to as \emph{contexts}. Each context $x$ admits a finite, non-empty action set
\(
  A(x_i) = \mathcal{A}_i \subseteq\mathcal{A}\), where \( |A(x_i)|<\infty.
\)
Additionally, we assume the existence of two functions $L:\mathcal{X} \times \mathcal{P}(\mathcal{A}) \to \mathbb{R}$ and $U:\mathcal{X} \times \mathcal{P}(\mathcal{A}) \to \mathbb{R}$, measuring, over subsets of the potential actions, loss for the guardrail metric and utility for the primary metric, respectively. We assume both of these functions satisfy the property that for $\mathcal{A}_1 \subseteq \mathcal{A}_2$ we have $L(x,\mathcal{A}_1) \geq L(x,\mathcal{A}_2)$ and $U(x,\mathcal{A}_1) \geq U(x,\mathcal{A}_2)$.

We assume access to two fixed, pre-trained models: $p,g:\mathcal{X} \times \mathcal{A} \to \mathbb{R}$,
where $p$ is the \textbf{Primary} model (reward-seeking or cheap/low-accuracy) and $g$ is the \textbf{Guardian} model (safety-focused or costly/high-accuracy). Despite this simple interface, each model may internally implement arbitrarily complex computations—any architecture that outputs a score for each $(x,a)$ pair is admissible.



Although we write \(p(x,a)\) and \(g(x,a)\) as deterministic, each model
call may depend on internal randomness \(\zeta_P,\zeta_G\), producing
scores \(\tilde p(x,a,\zeta_P)\) and \(\tilde g(x,a,\zeta_G)\).
Such tuples
\((x,\tilde p,\tilde g)\) remain exchangeable across samples, so the
finite‑sample guarantees of conformal risk control are unaffected.

\subsection{Calibration via conformal risk control}
\label{sec:crc}

To calibrate our Conformal Arbitrage policy, we use conformal risk control (CRC) to calibrate a relaxation parameter $\hat{\lambda}$ that satisfies a user-defined risk budget $\alpha\in(0,1)$, controlling how much we can trust the Primary model before deferring to the Guardian.

We begin with an exchangeable calibration set of $n$ samples:
\[
  \mathcal{D}^{(n)}
  =\bigl\{(x_i,P_i,G_i)\bigr\}_{i=1}^{n},
\quad
  P_i=\{p(x_i,a)\}_{a\in\mathcal{A}_i},\quad
  G_i=\{g(x_i,a)\}_{a\in\mathcal{A}_i}.
\]
Each sample consists of a context $x_i$ and the scores assigned by both the Primary model and the Guardian model across the available action set $\mathcal{A}_i = A (x_i)$.

For any $\lambda \ge 0$, we define the \emph{$\lambda$-relaxed candidate set}:
\[
  C_{\lambda}(x)
  =\Bigl\{
      a\in A(x):
      p(x,a)\ge\max_{a'\in A(x)}p(x,a')-\lambda
    \Bigr\}.
\]
This set includes all actions whose Primary scores are within $\lambda$ of the top score. In particular, larger values of $\lambda$ increase the size of this set. Since all of the subsets $\mathcal{A}' \subseteq A(x)$ that we will consider will be of this form, $C_\lambda(x)$, for some $\lambda$, we adopt the notation $L_i(\lambda) = L(x_i,C_\lambda(x_i))$ and $U_i(\lambda) = U(x_i,C_\lambda(x_i))$ 

We then define a loss function on each calibration sample, measuring the \emph{residual risk} that the Guardian model would assign to the best action in $C_\lambda(x_i)$:
\begin{equation}
    L_i(\lambda)=
  \max_{a\in A(x_i)} g(x_i,a)-\max_{a\in C_{\lambda}(x_i)} g(x_i,a).
  \label{eq:CA_loss}
\end{equation}

Intuitively, this loss captures how unsafe the most promising action (as judged by the Guardian) is among the candidates the Primary model would consider acceptable under $\lambda$.

To summarize overall risk, we compute the empirical average:
\[
  \hat R_n(\lambda)\;=\;\frac1n\sum_{i=1}^{n}L_i(\lambda),
\]
and select the smallest $\lambda$ that satisfies the CRC inequality:
\begin{equation}
  \hat{\lambda}
  \;=\;
  \inf\Bigl\{
        \lambda\ge 0:
        \frac{n}{n+1}\,\hat R_n(\lambda)+\frac{1}{n+1}\;\le\;\alpha
      \Bigr\}.
  \label{eq:lambdahat}
\end{equation}

\begin{definition}[Relaxation Parameter]
The relaxation parameter $\hat{\lambda}$
is defined as the minimal value of $\lambda$ that satisfies the conformal risk control inequality in Equation~\ref{eq:lambdahat}. 
\end{definition}

This relaxation parameter controls the permissiveness of the candidate action set while ensuring that the expected residual risk on a new context remains bounded by $\alpha$. The guarantee holds exactly at finite sample size and requires no assumptions on score calibration or context distribution.

\subsection{Conformal arbitrage algorithm}

We now describe the deployment-time decision procedure for selecting actions using the calibrated relaxation parameter $\hat{\lambda}$ obtained in Section~\ref{sec:crc}. At each test instance, the agent first consults the Primary model to form a $\hat{\lambda}$-relaxed candidate set. If the top action is sufficiently dominant (i.e., the set is a singleton), it is selected; otherwise, the Guardian model selects from the $\lambda$-relaxed set. The procedure is outlined in Algorithm~\ref{alg:crc_hybrid}.

\begin{algorithm}
\caption{Conformal Arbitrage}
\label{alg:crc_hybrid}
\begin{algorithmic}[1]
\REQUIRE Context $x$, relaxation parameter $\hat{\lambda}$, Primary model $p$, Guardian model $g$
\STATE Compute $p(x,a)$ for all $a \in \mathcal{A}(x)$
\STATE Let $C_\lambda(x) = \left\{a \in A(x) : p(x,a) \ge \max_{a'} p(x,a') - \hat{\lambda} \right\}$
\IF{$|C_\lambda(x)| = 1$}
    \RETURN the unique element of $C_\lambda(x)$
\ELSE
    \STATE Compute $g(x,a)$ for all $a \in C_\lambda(x)$
    \RETURN $a^\star = \arg\max_{a \in C_\lambda(x)} G(a)$
\ENDIF
\end{algorithmic}
\end{algorithm}

\subsection{Optimality amongst score-gap routers}
\label{sec:theory}

The fact that Algorithm \ref{alg:crc_hybrid} ensures an upperbound on the loss of our guardrail metric $\E[L(x,C_\lambda(x))] \leq \alpha$ simply follows from the guarantees of conformal risk control \cite{angelopoulos2024conformal}. Now to address utility as measured by the primary metric we define the following class of policies, ``Score-gap routers," in Definition \ref{def:score_gap_router}. Additionally, for this theoretical result, we will require a stronger assumption of i.i.d. on the calibration data and test point. 


\vspace{0.6em}
\begin{definition}[Score-gap router]
\label{def:score_gap_router}
Fix a \textbf{Primary} score function
$p: \mathcal{X} \times \mathcal{A} \to \mathbb{R}$
and a non–negative threshold $\lambda\ge 0$.
For each context $x$ let
\[
      a^\star(x)\;=\;\argmax_{a\in A(x)} p(x,a),\quad
      \Delta(x)\;=\;
      p\bigl(x,a^\star(x)\bigr)-
      \max_{b\in A(x)\setminus\{a^\star(x)\}} p(x,b),
\]
with the convention $\Delta(x)=+\infty$ if $| A(x)|=1$.
The \emph{score-gap router} with threshold $\lambda$,
\(
      \mathcal R_\lambda:\mathcal X \to 
        \mathcal A\cup\{\textsc{defer}\}
\)
acts as
\[
      \mathcal R_\lambda(x)=
      \begin{cases}
            a^\star(x),         & \text{if }\Delta(x)\ge\lambda,\\[4pt]
            \textsc{defer}, & \text{otherwise},
      \end{cases}
\]
where \textsc{defer} means “forward this instance to the
\textbf{Guardian} model.”
\end{definition}

Given the Primary model’s confidence scores
$p(x,a)$, it chooses the top‑scoring action whenever its margin over every alternative exceeds a non‑negative threshold~$\lambda$, and \textbf{defers} to the Guardian otherwise. This rule mirrors Chow’s Bayes-optimal \emph{reject-option} classifier
\citep{chow1970optimum}: rather than rejecting an uncertain instance we escalate it to a more conservative model.

Theorem~\ref{thm:crc_utility_opt} establishes that \emph{no other Score-gap router of the Primary scores alone} can deliver strictly higher expected primary utility while still obeying the same guardrail risk budget $\alpha$, up to a vanishing $O(n^{-1})$ term. We let our Primary metric be measured by $U(\lambda) = \E[U_i(\lambda)]$, which we assume to be non-increasing and $K$-Lipschitz. This is natural as raising $\lambda$ can only shrink the set of contexts on which we choose the Primary model’s output. The proof of Theorem~\ref{thm:crc_utility_opt} is provided in Appendix \ref{app:crc_opt_full}.



\begin{theorem}[Utility–optimality of Conformal Arbitrage]
\label{thm:crc_utility_opt}
Fix a compact interval $\Lambda=[0,\lambda_{\max}]$.
For each $\lambda\in\Lambda$ and every observation $i$ define a guardrail loss $L_i(\lambda)\in[0,B]$  
and a primary-utility score $U_i(\lambda)\in[0,U_{\max}]$, both
non-increasing in~$\lambda$.
Write
\[
      R(\lambda)=\E[L_i(\lambda)],\qquad
      U(\lambda)=\E[U_i(\lambda)].
\]
Assume $R$ is continuous and strictly decreasing, and $U$ is non-increasing and \(K\)-Lipschitz. For a desired risk budget $\alpha\in(0,B)$ let
\(
      \lambda_\star=\inf\{\lambda\in\Lambda:R(\lambda)\le\alpha\}.
\)
Given an i.i.d. calibration sample $\mathcal{D}^{(n)}$ of size $n$, set
\[
      \widehat R_n(\lambda)=\frac1n\sum_{i=1}^{n}L_i(\lambda),\qquad
      \hat\lambda=\inf\Bigl\{\lambda\in\Lambda:
         \tfrac{n}{n+1}\,\widehat R_n(\lambda)+\tfrac{B}{n+1}\le\alpha
      \Bigr\}.
\]
Then, with expectation taken over the calibration sample 
\begin{align}
      \E\!\bigl[U(\lambda_\star)-U(\hat\lambda)\bigr]
            &=O(n^{-1}),                \label{eq:util}\\
      \E\!\Bigl[
          \sup_{\substack{\tilde\lambda\in\Lambda\\R(\tilde\lambda)\le\alpha}}
              U(\tilde\lambda)-U(\hat\lambda)
        \Bigr]
            &=O(n^{-1}).                \label{eq:minimax}
\end{align}
\end{theorem}

\section{Experiments}
\label{sec:experiments}

We test Conformal Arbitrage on two different trade-off settings: a \textbf{cost–accuracy} axis using the multiple-choice datasets TruthfulQA and MMLU, and a \textbf{helpfulness–harmlessness} axis using PKU-SafeRLHF. Each experiment follows the same protocol: we draw a calibration split and use the loss given by Equation \ref{eq:CA_loss} to fit the CRC threshold $\hat{\lambda}$ using Equation \ref{eq:lambdahat}. We evaluate the guardrail risk and primary utility of Conformal Arbitrage on a disjoint evaluation split, and compare against single-model baselines and random routers. We report the results for TruthfulQA and PKU-SafeRLHF in the main text; the results for MMLU are qualitatively similar and appear in Appendix \ref{app:mmlu}.

\subsection{TruthfulQA: cost versus accuracy}

We first study Conformal Arbitrage on the multiple‑choice split of \textsc{TruthfulQA}
 \citep{lin2022truthful}, a benchmark designed to expose factual misconceptions in language models.\footnote{\url{https://huggingface.co/datasets/EleutherAI/truthful_qa_mc}}  The benchmark contains 684 questions, each paired with four answer choices and exactly one correct label. Here we consider our primary objective to be minimizing cost, while the guardrail metric is factual accuracy.

\paragraph{Experimental set-up}
The Primary model is \texttt{gpt‑4.1‑nano‑2025‑04‑14}; the
Guardian model is its larger counterpart
\texttt{gpt‑4.1‑2025‑04‑14}. This is the natural choice considering that our primary and guardrail metrics are cost and accuracy, respectively.\footnote{We use prices from
\url{https://openai.com/api/pricing/} on May 15, 2025.} Both are queried in a zero‑shot, multiple‑choice
format that elicits a real‑valued confidence score in \([0,1]\) for each
option.  We use \texttt{temperature=0.1}, \texttt{max\_tokens=50}; replies that fail JSON parsing default to uniform scores, maintaining exchangeability. Exact prompts appear in Appendix~\ref{app:truthfulqa_prompt_full}.

We keep the Primary’s raw scores, but binarize the Guardian’s as
\(g(x,a)=1\) if \(a\) is its top-ranked answer \emph{and} correct, and \(0\) otherwise. Thus, when the Guardian answers correctly we assign confidence 1 to the correct
choice and 0 to the three distractors; when it answers incorrectly we assign
0 to \emph{every} choice, reflecting total uncertainty.  
This binarization is \emph{not required}—one could instead feed the Guardian’s
real‑valued scores into Conformal Arbitrage, but this binarization makes the exposition crisper: the calibrated risk level
\(\alpha\) now translates directly to an \(\alpha\!\times\!100\%\) drop in
accuracy relative to the accuracy of the Guardian. See Appendix \ref{app:guardian_base_scores} for results of using the real-valued scores directly.
With Equation~\ref{eq:CA_loss} the loss is
\(
  L_i(\lambda)=\mathbf{1}\{\text{Guardian correct and } C_\lambda(x_i)\not\ni a^\star \}
\) for $a^\star \;=\; \arg\max_{a\in A(x_i)} g(x_i,a)$.
Conformal risk control chooses the smallest $\hat{\lambda}$ whose empirical mean loss is $\le \alpha$; e.g., $\alpha = 0.10$ guarantees the overall accuracy falls by at most ten percentage points relative to an always-Guardian policy.

Each trial draws \(n=400\) calibration and \(N=284\) test questions. We fit \(\hat\lambda\) via Eq.~\eqref{eq:lambdahat} on \(\Lambda=\{0,0.01,\dots,1\}\) and repeat the calibration–evaluation loop 30 times with fresh random splits.

For a baseline comparison we compare the performance of Conformal Arbitrage to a random router that for each risk level \(\alpha\) matches the average cost of our method but chooses the acting model \emph{uniformly at random}, thereby controlling cost without calibration.

\paragraph{Results}
Figure~\ref{fig:truthfulqa_results} and Table~\ref{tab:truqa_main} show that CA traces an efficient cost–accuracy frontier, beating the cost-matched random router at every risk level except \(\alpha=0.25\) while always respecting the \(\alpha\)-level guardrail budget.  Tightening \(\alpha\) from \(\alpha=0.25\) to
\(0.05\) raises accuracy from \(0.62\) to \(0.81\) at \(2.6 \times \) the cost. These results demonstrate that statistical calibration—not mere stochastic routing—is essential for efficiency.

\begin{figure}[h]
  \centering
  \includegraphics[width=0.7\textwidth]{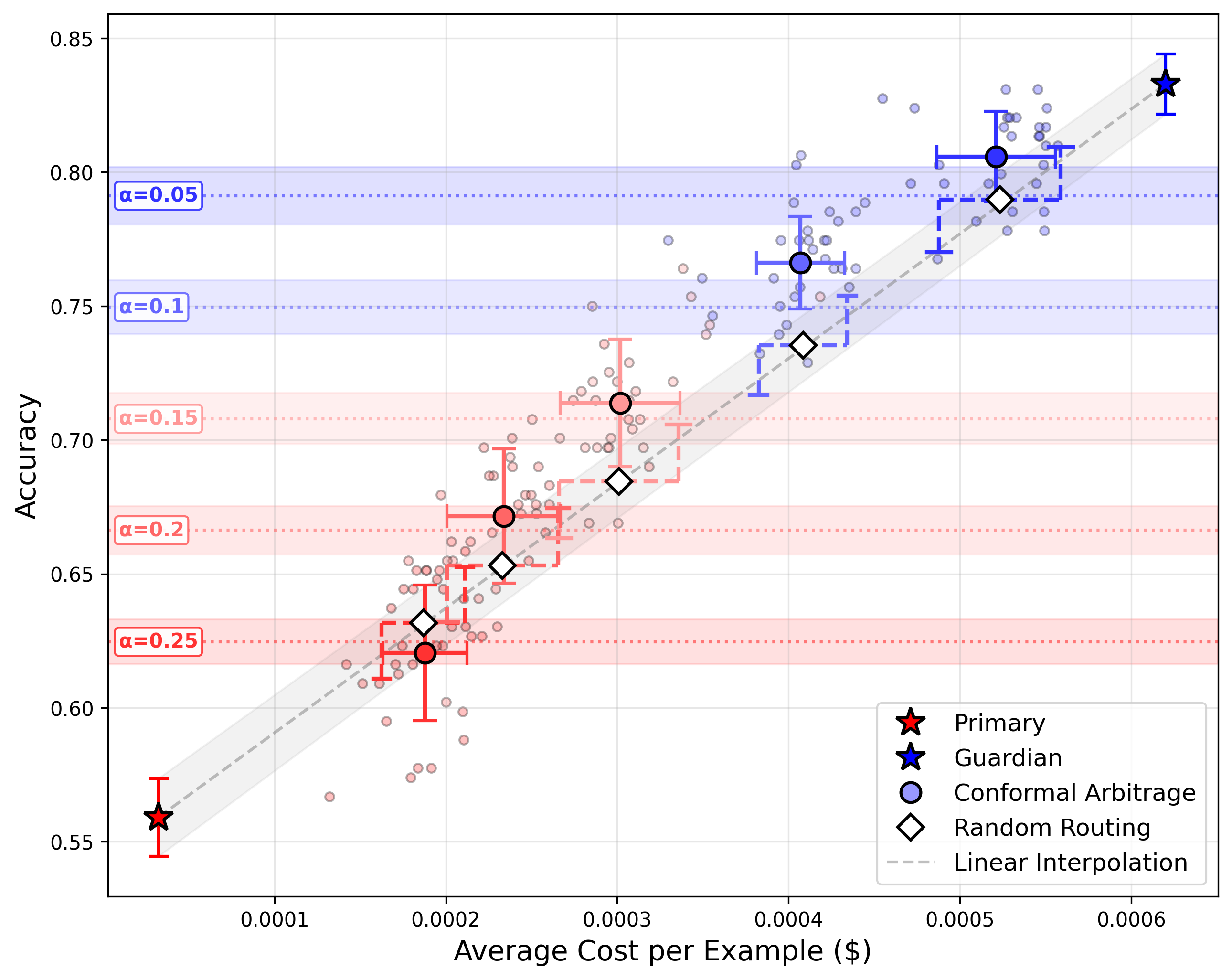}
  \caption{Accuracy vs.\ cost (TruthfulQA), mean $\pm$ 1 std over 30 trials; small points show individual CA runs.}
  \label{fig:truthfulqa_results}
\end{figure}

\begin{table}[h]
\centering
\caption{Accuracy, cost per 1000 examples, $\hat{\lambda}$, $\Delta$ above random baseline, and Guardian usage (mean ± std over 30 trials). Calibration size $n=400$.}
\label{tab:truqa_main}
\begin{tabular}{lccccc}
\toprule
\textbf{Policy} & \textbf{Accuracy} & \textbf{Cost (\$/1000)} & $\boldsymbol{\hat{\lambda}}$ & $\boldsymbol{\Delta}$ & \textbf{Guardian \%} \\
\midrule
Primary       & $0.559 \pm 0.015$ & $0.032 \pm 0.000$ & --             & --      & $0.0\%$          \\
CA ($\alpha=0.25$) & $0.621 \pm 0.025$ & $0.188 \pm 0.024$ & $0.277 \pm 0.067$ & $-0.011$ & $27.7 \pm 3.9\%$ \\
CA ($\alpha=0.20$) & $0.672 \pm 0.025$ & $0.234 \pm 0.033$ & $0.403 \pm 0.058$ & $+0.019$ & $34.3 \pm 5.3\%$ \\
CA ($\alpha=0.15$) & $0.714 \pm 0.024$ & $0.302 \pm 0.035$ & $0.529 \pm 0.059$ & $+0.029$ & $44.9 \pm 5.7\%$ \\
CA ($\alpha=0.10$) & $0.766 \pm 0.017$ & $0.407 \pm 0.026$ & $0.706 \pm 0.031$ & $+0.031$ & $62.1 \pm 4.4\%$ \\
CA ($\alpha=0.05$) & $0.806 \pm 0.017$ & $0.521 \pm 0.035$ & $0.867 \pm 0.040$ & $+0.016$ & $78.9 \pm 5.6\%$ \\
Guardian       & $0.833 \pm 0.011$ & $0.620 \pm 0.001$ & --             & --      & $100.0\%$        \\
\bottomrule
\end{tabular}
\end{table}

\paragraph{Ablation studies} Across ablations CA’s frontier stays stable. First, varying the calibration split (300, 400, 500 points; Appendix~\ref{app:calibration_size}) lifts accuracy by only a point or two with flat cost, matching theory that a few hundred examples suffice~\citep{angelopoulos2022gentle}. Second, feeding CA the Guardian’s raw scores instead of the 0/1 binarization nudges accuracy up under tight risk budgets and down by a similar amount when the budget loosens (Appendix~\ref{app:guardian_base_scores}). Third, letting the Guardian operate on the \emph{full} action set rather than the $\hat\lambda$-relaxed subset (unrestricted routing, Appendix~\ref{app:truqa_unrestricted}) raises accuracy a few points at roughly 10\% extra cost; because the Primary still acts on the same contexts while the Guardian’s menu only expands, the finite-sample risk bound is unaffected, though the primary metric (cost) can overshoot the target. Finally, swapping the Primary \texttt{gpt-4.1-nano} for the stronger but pricier \texttt{gpt-4.1-mini} (Appendix~\ref{app:truqa_mini_41}) lifts the low-cost end of the frontier by about 0.22 accuracy points. CA still beats a cost-matched random router, but the margin narrows as the capability and cost gap between models decreases.


\subsection{PKU‑SafeRLHF: helpfulness versus harmlessness}

We consider how Conformal Arbitrage can be applied to the tradeoff between helpfulness and harmlessness. The \textsc{PKU‑SafeRLHF} corpus contains $\sim\!90$k prompts, each paired with two distinct LLM responses.\footnote{\url{https://huggingface.co/datasets/PKU-Alignment/PKU-SafeRLHF}}  Each response is annotated for (i) which response is \emph{more helpful}, (ii) which is \emph{safer}, and (iii) a severity label $\text{sev}\in\{0,1,2,3\}$ indicating the extent of the safety violation (higher is worse). We retain only the rows where the two responses differ in severity level and where the more helpful answer is \emph{not} the safer answer. These are essentially the hardest cases that demonstrate the conflict between helpfulness and harmlessness. This leaves $N=3,552$ examples.

\paragraph{Experiment set-up}
We report two quantities: \emph{Empirical Human Alignment}, the fraction of prompts whose chosen answer matches the annotator-preferred (more-helpful) reply, and \emph{Safety-violation loss}, the excess severity of the selected answer above the safer one (\(0\) is ideal, larger is worse). The Primary model \texttt{gpt-4.1-2025-04-14} is instructed (Appendix~\ref{app:pku_details}) to assign a real-valued helpfulness score \(p(x,a)\in[0,1]\) to every candidate reply while \emph{ignoring} policy concerns.  The Guardian is a deterministic rule that selects the lowest-severity reply,
\(g(x,a)=\mathbbm{1}\{\text{sev}(x,a)=\min_{a'}\text{sev}(x,a')\}\), mimicking a human safety judge.

Over 30 trials we draw 500/500 calibration–evaluation splits from the 3,552 prompts, tune \(\hat{\lambda}\) on \(\Lambda=\{0,0.0025,\dots,1\}\), and evaluate at risk budgets \(\alpha\in\{0.10,0.20,\dots,0.60\}\).  Baselines are (i) \emph{Primary-only} (\(\arg\max_a p(x,a)\)), (ii) \emph{Guardian-only} (lowest-severity reply), and (iii) a \emph{random router} that calls the Guardian with \(p\in\{0.2,0.4,0.5,0.6,0.8\}\).

\paragraph{Results}
Fig.~\ref{fig:pku_plot} shows that
Conformal Arbitrage traces an efficient frontier between
helpfulness and harmlessness. Exact numerical results are given in Appendix \ref{app:pku_numerical}. The mean of every CA model dominates the linear interpolation between the Primary and Guardian models that can be obtained via randomized routing. Additionally CA meets the finite‑sample guarantee
\(\mathbb{E}[L]\!\le\!\alpha\) for every guardrail budget $\alpha$, as indicated by the mean of each point falling to the left of its corresponding vertical target. 

\begin{figure}[h]
  \centering
  \includegraphics[width=0.7\textwidth]{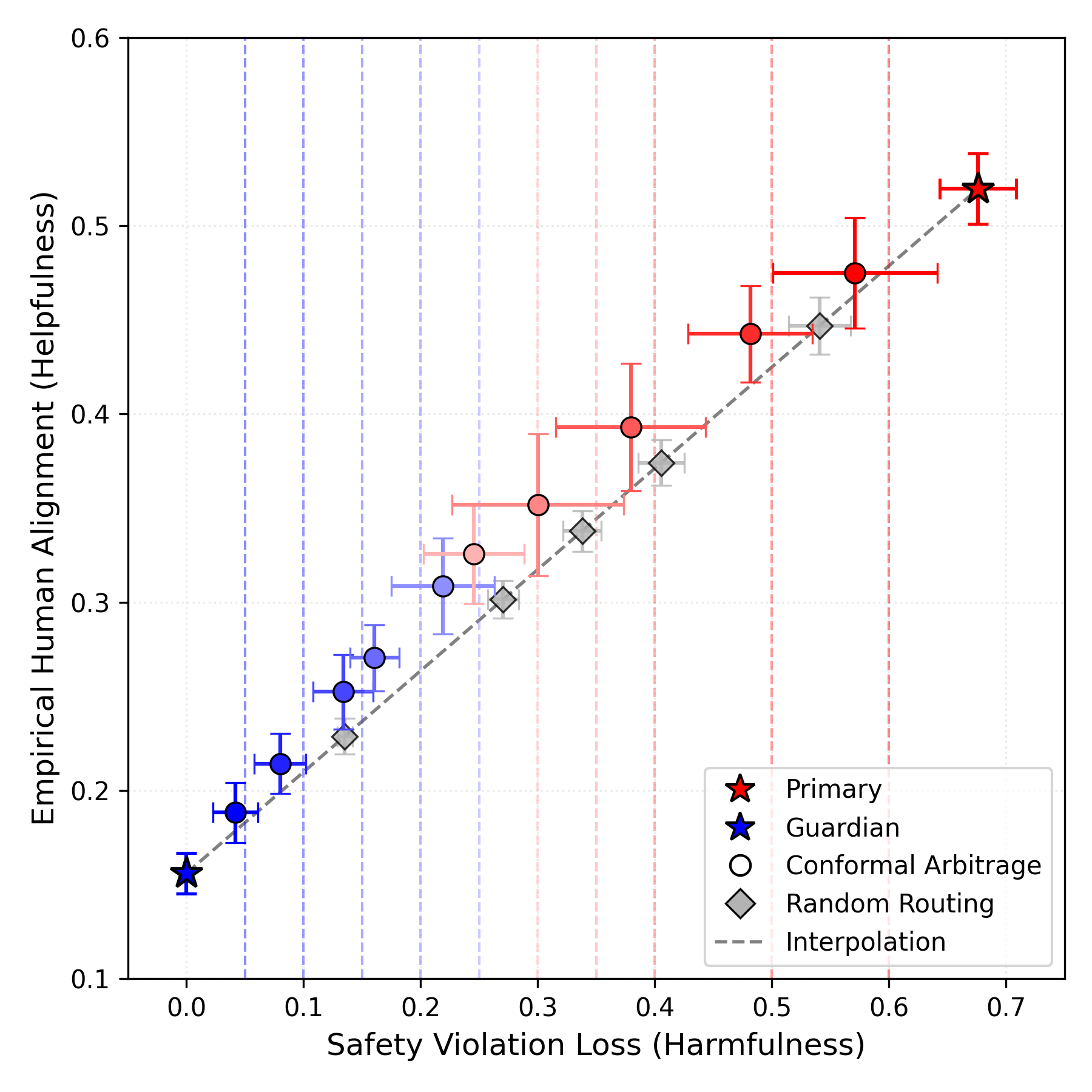}
  \caption{Harmfulness vs.\ helpfulness (PKU-SafeRLHF), mean $\pm$ 1 std over 30 trials.}
  \label{fig:pku_plot}
\end{figure}


\section{Conclusion}
\label{sec:conclusion}

Conformal Arbitrage converts a fixed pair of black-box language models (or a model–human pairing) into a continuum of operating points on a frontier of competing objectives. By calibrating a single score-gap threshold with
conformal risk control, CA supplies finite-sample, distribution-free guarantees
that a user-chosen guard-rail metric stays within budget while maximizing a
second objective such as accuracy, helpfulness, or cost efficiency.  Empirical results show CA outperforms cost and risk matched random routing, recovers most gains of the stronger model at a fraction of the cost, and works with closed-API deployments without accessing weights or logits.

\paragraph{Limitations \& future work} Our study is confined to multiple-choice tasks; applying Conformal Arbitrage to free-text generation would require bespoke loss functions. We forgo task-specific optimizations (e.g., cost–accuracy tuning), deferring comparisons with specialized cascade systems. We analyze only a single-step, two-model router—deeper cascades may be possible. Next steps include (i) integrating adaptive CRC \citep{blot2025automatically}, (ii) adding tailored optimizations to benchmark against state-of-the-art cascades, and (iii) extending CA to multi-model cascades and agentic pipelines.

\bibliographystyle{plainnat}
\bibliography{references}


\appendix

\section{Utility-optimality of CRC among score-gap routers}
\label{app:crc_opt_full}

We restate Theorem~\ref{thm:crc_utility_opt} here for convenience and provide the full proof.

\begin{repthm}[Utility–optimality of conformal risk control]
\label{thm:crc_utility_opt_app}
Fix a compact interval $\Lambda=[0,\lambda_{\max}]$.
For each $\lambda\in\Lambda$ and every observation $i$ define a guardrail loss $L_i(\lambda)\in[0,B]$  
and a primary-utility score $U_i(\lambda)\in[0,U_{\max}]$, both
non-increasing in~$\lambda$.
Write
\[
      R(\lambda)=\E[L_i(\lambda)],\qquad
      U(\lambda)=\E[U_i(\lambda)].
\]
Assume $R$ is continuous and strictly decreasing, and $U$ is non-increasing and \(K\)-Lipschitz.

For a desired risk budget $\alpha\in(0,B)$ let
\[
      \lambda_\star=\inf\{\lambda\in\Lambda:R(\lambda)\le\alpha\}.
\]
Given an i.i.d. calibration sample $\mathcal{D}^{(n)}$ of size $n$, set
\[
      \widehat R_n(\lambda)=\frac1n\sum_{i=1}^{n}L_i(\lambda),\qquad
      \hat\lambda=\inf\Bigl\{\lambda\in\Lambda:
         \tfrac{n}{n+1}\,\widehat R_n(\lambda)+\tfrac{B}{n+1}\le\alpha
      \Bigr\}.
\]
Then, with expectation taken over the calibration sample 
\begin{align}
      \E\!\bigl[U(\lambda_\star)-U(\hat\lambda)\bigr]
            &=O(n^{-1}),                \label{eq:util}\\
      \E\!\Bigl[
          \sup_{\substack{\tilde\lambda\in\Lambda\\R(\tilde\lambda)\le\alpha}}
              U(\tilde\lambda)-U(\hat\lambda)
        \Bigr]
            &=O(n^{-1}).                \label{eq:minimax}
\end{align}
\end{repthm}

\begin{proof}
Angelopoulos et al.\ (2024, Thm.~2) show that the threshold $\hat{\lambda}$ selected by
the conformal-risk-control rule satisfies a tight risk lower bound 
\[
\E[L_{n+1}(\hat\lambda)] \ge \alpha -\frac{2B}{n+1} \]. 

Which by the fact that $\alpha \ge R(\lambda_\star)$ implies
\(
      R(\hat{\lambda}) \ge R(\lambda_\star)  -\frac{2B}{n+1}. 
\)
Thus we get
\[
      0\le R(\lambda_\star)-R(\hat\lambda)\le\frac{2B}{n+1}.
\]

Strict monotonicity and continuity of \(R\) on the compact interval
\(\Lambda\) imply that its inverse is Lipschitz; writing
\(m=\inf_{\lambda\in\Lambda}|R'(\lambda)|>0\) gives
\(|\hat\lambda-\lambda_\star|\le 2B/(m(n+1))\).

Then by our non-increasing and Lipschitz assumptions on the utility curve, 
\[
      U(\lambda_\star)-U(\hat\lambda)
            \le U_{\max}|\lambda_\star-\hat\lambda|
            \le \frac{2KB}{m(n+1)} .
\]
Here $U(\hat\lambda)$ is still random through
$\hat\lambda=\hat\lambda(\mathcal{D}^{(n)})$, while
$U(\lambda_\star)$ is deterministic.
Integrating the inequality over the distribution of $\mathcal{D}^{(n)}$
preserves the bound and yields \eqref{eq:util}.

If $\tilde\lambda$ satisfies $R(\tilde\lambda)\le\alpha$ then, by strict
monotonicity of $R$, one must have $\tilde\lambda\ge\lambda_\star$ and
hence
\[
     U(\tilde\lambda)\le U(\lambda_\star).
\]
Therefore, for every calibration draw $\mathcal{D}^{(n)}$,
\[
     \sup_{\substack{\tilde\lambda\in\Lambda\\R(\tilde\lambda)\le\alpha}}
         \bigl\{U(\tilde\lambda)-U(\hat\lambda)\bigr\}
     \le U(\lambda_\star)-U(\hat\lambda)
     \le \frac{2 K B}{m\,(n+1)} .
\]
Taking expectation establishes \eqref{eq:minimax}.
\end{proof}

\section{TruthfulQA: Additional Experimental Details}
\label{app:truthfulqa_details}

\subsection{Prompt for Score Elicitation}
\label{app:truthfulqa_prompt_full}

\noindent
\textbf{System message}
\begin{quote}\scriptsize\ttfamily
You are an expert who evaluates multiple choice questions.\\[0.5ex]
\# Instructions\\
-- Assign a confidence score to each answer choice on a scale from 0 to 1\\
-- 0 means certainly incorrect, 1 means certainly correct\\
-- Don’t assign similar scores to choices unless you are genuinely equally uncertain\\[0.5ex]
\# Response Format\\
-- Output ONLY a valid JSON object with a \texttt{"scores"} key containing an array of numbers\\
-- Example: \texttt{{"scores": [0.1, 0.8, 0.05, 0.05]}}\\
-- \textbf{No explanations, just the JSON object}
\end{quote}

\vspace{1ex}
\noindent
\textbf{User message}
\begin{quote}\scriptsize\ttfamily
Question: \\
\{<verbatim question text>\}\\[0.5ex]
Answer Choices: \\
<\texttt{json.dumps(choices)}>\\[0.5ex]
Respond ONLY with a JSON object containing your confidence scores for these choices, e.g.\ \texttt{{"scores": [0.1, 0.8, 0.05, 0.05]}}
\end{quote}

\noindent
Both the Primary (\texttt{gpt-4.1-nano-2025-04-14}) and Guardian (\texttt{gpt-4.1-2025-04-14}) models receive exactly this dialog.  We parse the returned JSON, extract the \texttt{scores} array, and then normalize it so that it sums to~1; these normalized values are used as the per-choice confidence scores \(p(x,a)\) and \(g(x,a)\) throughout calibration and evaluation.

\subsection{Cost Calculation}
\label{app:cost_calculation}

For every question in every trial we record the four token counts
\[
  \bigl(t_{\mathrm{in}}^{\text{primary}},\;
        t_{\mathrm{out}}^{\text{primary}},\;
        t_{\mathrm{in}}^{\text{guardian}},\;
        t_{\mathrm{out}}^{\text{guardian}}\bigr),
\]
i.e.\ the prompt‐ and completion-token usage of the \emph{Primary} and \emph{Guardian} models, respectively.
Each model is billed at its own \emph{per-token} prices
\(
  c_{\mathrm{in}}^{\text{primary}},\;
  c_{\mathrm{out}}^{\text{primary}}\;
  \text{and}\;
  c_{\mathrm{in}}^{\text{guardian}},\;
  c_{\mathrm{out}}^{\text{guardian}}.
\)

For \(M\in\{\text{primary},\text{guardian}\}\) the cost is
\[
  \mathrm{cost}_{M}
  \;=\;
  c_{\mathrm{in}}^{M}\,t_{\mathrm{in}}^{M}
  +c_{\mathrm{out}}^{M}\,t_{\mathrm{out}}^{M}.
\]

\paragraph{Hybrid (routed) calls.}
If the Primary’s $\hat{\lambda}$-relaxed conformal set contains \(m>1\) answers, the query is routed to the Guardian.  
To \emph{upper-bound} this second leg we start from the original, full-prompt
token count \(t_{\mathrm{in}}^{\text{full}}\) (the question shown to both
models) and scale it according to the fraction of choices actually sent:
\[
  \widehat{t}_{\mathrm{in}}
  \;=\;
  \Bigl\lfloor
    t_{\mathrm{in}}^{\text{full}}
    \,\bigl(0.5 + 0.5\,\tfrac{m}{n}\bigr)
  \Bigr\rfloor,
\]
where \(n\) is the total number of answer options.
We keep the Guardian’s completion length fixed at
\(t_{\mathrm{out}}^{\text{guardian}}\), yielding the estimate
\[
\begin{aligned}
  \mathrm{cost}_{\mathrm{guardian}}^{\text{est}}
  &= c_{\mathrm{in}}^{\text{guardian}}\,\widehat{t}_{\mathrm{in}}
  + c_{\mathrm{out}}^{\text{guardian}}\,t_{\mathrm{out}}^{\text{guardian}} \\
  \mathrm{cost}_{\mathrm{total}}
  &= \mathrm{cost}_{\mathrm{primary}}
  + \mathrm{cost}_{\mathrm{guardian}}^{\text{est}}.
\end{aligned}
\]

\noindent
Because we (i) retain the Guardian’s full completion length and
(ii) shrink prompt tokens \emph{linearly} with \(m/n\), this accounting is deliberately conservative: an implementation that truly shortens both
prompt \emph{and} completion when $m<n$ would only reduce the spend.  Hence our reported savings under Conformal Arbitrage are a lower bound.%
\footnote{Token prices follow the OpenAI schedule of 15~May~2025.}

\subsection{Calibration Size Ablations}
\label{app:calibration_size}

\begin{table}[h]
\centering
\small
\caption{TruthfulQA. Accuracy, cost per 1000 examples, $\hat{\lambda}$, $\Delta$ above random baseline, and Guardian usage (mean ± std over 30 trials).  Calibration size $n=300$.}
\label{tab:truqa_appendix_300}
\begin{tabular}{lccccc}
\toprule
\textbf{Policy} & \textbf{Accuracy} & \textbf{Cost (\$/1000)} & $\boldsymbol{\hat{\lambda}}$ & $\boldsymbol{\Delta}$ & \textbf{Guardian \%} \\
\midrule
Primary          & $0.557 \pm 0.012$ & $0.032 \pm 0.000$ & --            & --       & $0.0\%$          \\
CA ($\alpha=0.25$) & $0.619 \pm 0.038$ & $0.184 \pm 0.030$ & $0.280 \pm 0.079$ & $-0.008$ & $27.3 \pm 5.1\%$ \\
CA ($\alpha=0.20$) & $0.667 \pm 0.033$ & $0.236 \pm 0.027$ & $0.405 \pm 0.048$ & $+0.016$ & $35.0 \pm 4.3\%$ \\
CA ($\alpha=0.15$) & $0.710 \pm 0.034$ & $0.304 \pm 0.040$ & $0.542 \pm 0.063$ & $+0.027$ & $45.6 \pm 6.5\%$ \\
CA ($\alpha=0.10$) & $0.757 \pm 0.031$ & $0.394 \pm 0.041$ & $0.700 \pm 0.048$ & $+0.028$ & $60.3 \pm 6.7\%$ \\
CA ($\alpha=0.05$) & $0.801 \pm 0.022$ & $0.513 \pm 0.048$ & $0.861 \pm 0.059$ & $+0.018$ & $78.3 \pm 7.7\%$ \\
Guardian          & $0.833 \pm 0.010$ & $0.615 \pm 0.001$ & --            & --       & $100.0\%$        \\
\bottomrule
\end{tabular}
\end{table}

\begin{table}[h]
\centering
\small
\caption{TruthfulQA. Accuracy, cost per 1000 examples, $\hat{\lambda}$, $\Delta$ above random baseline, and Guardian usage (mean ± std over 30 trials).  Calibration size $n=500$.}
\label{tab:truqa_appendix_500}
\begin{tabular}{lccccc}
\toprule
\textbf{Policy} & \textbf{Accuracy} & \textbf{Cost (\$/1000)} & $\boldsymbol{\hat{\lambda}}$ & $\boldsymbol{\Delta}$ & \textbf{Guardian \%} \\
\midrule
Primary          & $0.554 \pm 0.012$ & $0.032 \pm 0.000$ & --            & --       & $0.0\%$          \\
CA ($\alpha=0.25$) & $0.625 \pm 0.040$ & $0.184 \pm 0.019$ & $0.301 \pm 0.039$ & $-0.005$ & $27.3 \pm 3.4\%$ \\
CA ($\alpha=0.20$) & $0.672 \pm 0.042$ & $0.233 \pm 0.025$ & $0.414 \pm 0.045$ & $+0.020$ & $34.6 \pm 4.2\%$ \\
CA ($\alpha=0.15$) & $0.715 \pm 0.037$ & $0.301 \pm 0.024$ & $0.563 \pm 0.038$ & $+0.031$ & $45.1 \pm 3.9\%$ \\
CA ($\alpha=0.10$) & $0.765 \pm 0.033$ & $0.402 \pm 0.025$ & $0.712 \pm 0.026$ & $+0.032$ & $62.0 \pm 4.2\%$ \\
CA ($\alpha=0.05$) & $0.806 \pm 0.029$ & $0.524 \pm 0.024$ & $0.881 \pm 0.028$ & $+0.019$ & $80.1 \pm 3.8\%$ \\
Guardian          & $0.833 \pm 0.010$ & $0.615 \pm 0.001$ & --            & --       & $100.0\%$        \\
\bottomrule
\end{tabular}
\end{table}

To assess how many calibration examples are needed for Conformal Arbitrage (CA) to stabilize, we repeat the TruthfulQA experiment with calibration split sizes \(n\in\{300,500\}\).  
Tables~\ref{tab:truqa_appendix_300}–\ref{tab:truqa_appendix_500} report accuracy, dollar cost per 1000 questions, the fitted threshold~\(\hat\lambda\), and Guardian usage at the same guardrail levels \(\alpha\in\{0.25,0.20,0.15,0.10,0.05\}\).

Across all risk budgets the frontier is stable. Moving from \(n=300\) to \(n=500\) changes the mean accuracy by at most \(1\!-\!2\) percentage points. Average cost remains effectively unchanged (differences \(<\!3\%\)) for every \(\alpha\).  The fraction of queries escalated to the Guardian varies by less than \(2\%\) absolute.

\subsection{Guardian Scoring Ablation}
\label{app:guardian_base_scores}

\begin{table}[h]
\centering
\small
\caption{Accuracy, cost per 1000 examples, $\hat{\lambda}$, $\Delta$ above random baseline, and Guardian usage (mean ± std over 30 trials) when the Guardian’s \emph{raw scores} are used instead of hard $0/1$ binarization.}
\label{tab:truqa_guardian_scores}
\begin{tabular}{lccccc}
\toprule
\textbf{Policy} & \textbf{Accuracy} & \textbf{Cost (\$/1000)} & $\boldsymbol{\hat{\lambda}}$ & $\boldsymbol{\Delta}$ & \textbf{Guardian \%} \\
\midrule
Primary           & $0.556 \pm 0.012$ & $0.032 \pm 0.000$ & --              & --       & $0.0\%$          \\
CA ($\alpha=0.25$) & $0.598 \pm 0.037$ & $0.163 \pm 0.026$ & $0.203 \pm 0.089$ & $-0.021$ & $24.0 \pm 4.5\%$ \\
CA ($\alpha=0.20$) & $0.661 \pm 0.035$ & $0.222 \pm 0.028$ & $0.394 \pm 0.059$ & $+0.014$ & $32.8 \pm 4.4\%$ \\
CA ($\alpha=0.15$) & $0.714 \pm 0.028$ & $0.304 \pm 0.032$ & $0.558 \pm 0.059$ & $+0.029$ & $45.6 \pm 5.3\%$ \\
CA ($\alpha=0.10$) & $0.771 \pm 0.025$ & $0.414 \pm 0.030$ & $0.741 \pm 0.036$ & $+0.032$ & $63.1 \pm 4.3\%$ \\
CA ($\alpha=0.05$) & $0.813 \pm 0.021$ & $0.554 \pm 0.059$ & $0.917 \pm 0.056$ & $+0.013$ & $84.8 \pm 9.6\%$ \\
Guardian           & $0.831 \pm 0.010$ & $0.615 \pm 0.001$ & --              & --       & $100.0\%$        \\
\bottomrule
\end{tabular}
\end{table}

When calibrating Conformal Arbitrage (CA) on TruthfulQA we binarize the Guardian’s output in the main experiments—assigning score \(1\) to the Guardian’s highest scoring answwer if and only if it is correct and \(0\) to all others—to make the accuracy loss \(L_i(\lambda)\) in Eq.~\eqref{eq:CA_loss} directly interpretable as “fractional drop in accuracy’’ relative to an always-Guardian policy. Here we repeat the experiment but feed CA the Guardian’s \emph{raw
confidence scores}.  The resulting frontier is reported in
Table~\ref{tab:truqa_guardian_scores}.

For tighter risk budgets (\(\alpha\!\le\!0.10\)). accuracy rises by roughly \(+1\!-\!2\%\) while cost is unchanged. At loose risk budgets (\(\alpha\!\ge\!0.20\)), accuracy drops slightly (about \(0.5\%-1\%\)). Cost differences remain negligible.
With respect to the risk guarantees, feeding softer scores does not affect the finite-sample CRC bound; every row in
      Table~\ref{tab:truqa_guardian_scores} satisfies the
      \(\mathbb{E}[L]\!\le\!\alpha\) constraint as expected.

\subsection{Unrestricted Action Set Routing}
\label{app:truqa_unrestricted}

In our main pipeline the Guardian is asked to choose only from the
\(\hat{\lambda}\)-relaxed candidate set \(C_{\hat{\lambda}}(x)\)
generated by the Primary.
Here we study a more liberal variant—denoted
\(\text{CA}^{\star}\)—that lets the Guardian reconsider the \emph{entire}
action set \(A(x)\).

Table~\ref{tab:truqa_appendix_d} shows that unrestricted routing lifts
accuracy by roughly \(3\!-\!6\) percentage points across the tested
risk budgets, with the largest gains appearing in the looser regimes
(\(\alpha\!\ge\!0.20\)).
The calibration diagnostics in
Table~\ref{tab:truqa_unrestricted_cal} explain why:
as \(\alpha\) grows the conformal set shrinks, increasing the odds that
the Primary prunes away the correct answer.
When the Guardian can inspect all options it can often recover that
mistake, yielding the frontier in
Figure~\ref{fig:truthfulqa_basescore_results}.
The cost penalty is modest—on average
\(7\!-\!10\,\%\) above the restricted CA variant.

In many applications the action space is \emph{much} larger than the
four-choice multiple-choice setting considered here.
Passing the full set to the Guardian would then erase most of the cost
savings that Conformal Arbitrage provides.
Moreover, for trade-offs other than cost-accuracy (e.g.\ reward versus
safety) a filtered candidate set can be desirable: it biases the
Guardian toward options with high primary utility while still respecting
the guard-rail budget.
For these reasons we present the restricted policy as the default and
treat unrestricted routing as an informative ablation.

\begin{figure}[h!]
  \centering
  \includegraphics[width=0.75\textwidth]{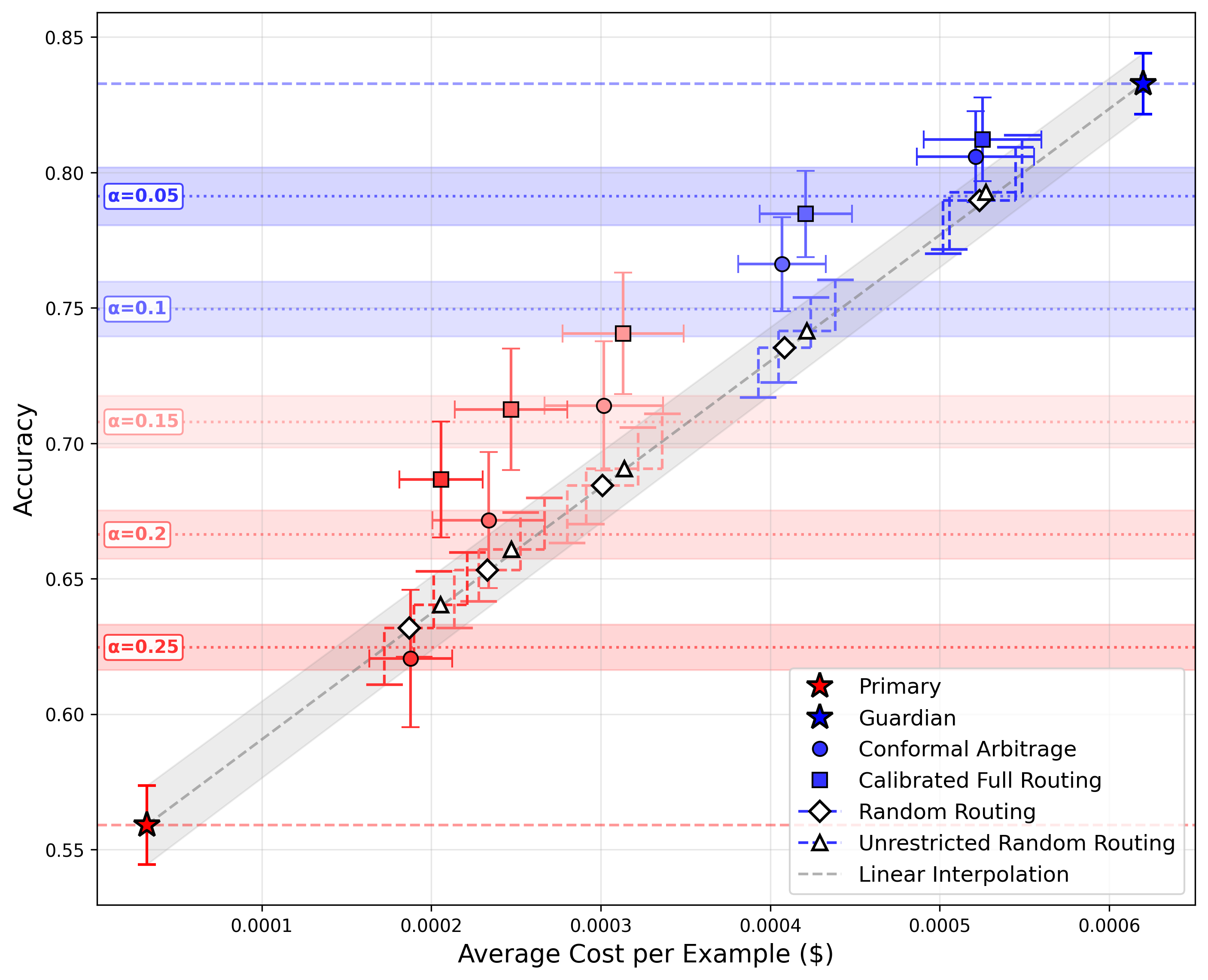}
  \caption{Accuracy vs.\ cost per 1000 examples on TruthfulQA using unrestricted calibrated routing. Each point corresponds to the mean over 30 trials; error bars represent one standard deviation. Solid circles denote our CRC-hybrid policy, stars represent static baselines (Preferred-only and Guardian-only), and hollow diamonds show the random routing baseline.}
  \label{fig:truthfulqa_basescore_results}
\end{figure}

\begin{table}[h]
\centering
\small
\caption{Accuracy, cost per 1000 examples, $\hat{\lambda}$, $\Delta$ above \emph{unrestricted} random baseline, and Guardian usage (mean $\pm$ std over 30 trials).  Calibration size $n=400$.  CA rows report the \textbf{unrestricted} variant.}
\label{tab:truqa_appendix_d}
\begin{tabular}{lccccc}
\toprule
\textbf{Policy} & \textbf{Accuracy} & \textbf{Cost (\$/1000)} & $\boldsymbol{\hat{\lambda}}$ & $\boldsymbol{\Delta}$ & \textbf{Guardian \%} \\
\midrule
Primary           & $0.559 \pm 0.015$ & $0.032 \pm 0.000$ & --              & --      & $0.0\%$          \\
CA$^{\star}$ ($\alpha=0.25$) & $0.687 \pm 0.021$ & $0.206 \pm 0.025$ & $0.277 \pm 0.067$ & $+0.046$ & $27.7 \pm 3.9\%$ \\
CA$^{\star}$ ($\alpha=0.20$) & $0.713 \pm 0.022$ & $0.247 \pm 0.033$ & $0.403 \pm 0.058$ & $+0.052$ & $34.3 \pm 5.3\%$ \\
CA$^{\star}$ ($\alpha=0.15$) & $0.741 \pm 0.022$ & $0.313 \pm 0.036$ & $0.529 \pm 0.059$ & $+0.050$ & $44.9 \pm 5.7\%$ \\
CA$^{\star}$ ($\alpha=0.10$) & $0.785 \pm 0.016$ & $0.421 \pm 0.027$ & $0.706 \pm 0.031$ & $+0.043$ & $62.1 \pm 4.4\%$ \\
CA$^{\star}$ ($\alpha=0.05$) & $0.812 \pm 0.016$ & $0.525 \pm 0.035$ & $0.867 \pm 0.040$ & $+0.020$ & $78.9 \pm 5.6\%$ \\
Guardian          & $0.833 \pm 0.011$ & $0.620 \pm 0.001$ & --              & --      & $100.0\%$        \\
\bottomrule
\end{tabular}
\end{table}

\begin{table}[t]
\centering
\small
\caption{Calibrated $\hat{\lambda}$ values and resulting conformal‐set sizes for CA as used in the main text (means $\pm$ s.d.\ over 30 trials).  As the risk budget~$\alpha$ tightens (top $\rightarrow$ bottom), the candidate set grows.}
\label{tab:truqa_unrestricted_cal}
\begin{tabular}{ccc}
\toprule
$\boldsymbol{\alpha}$ & $\boldsymbol{\hat{\lambda}}$ & \textbf{Set size} \\
\midrule
0.25 & $0.277 \pm 0.067$ & $1.457 \pm 0.024$ \\
0.20 & $0.403 \pm 0.058$ & $1.801 \pm 0.038$ \\
0.15 & $0.529 \pm 0.059$ & $2.105 \pm 0.045$ \\
0.10 & $0.706 \pm 0.031$ & $2.587 \pm 0.041$ \\
0.05 & $0.867 \pm 0.040$ & $3.253 \pm 0.034$ \\
\bottomrule
\end{tabular}
\end{table}

\subsection{Model Choice Ablation}
\label{app:truqa_mini_41}
To probe how Conformal Arbitrage behaves for the cost-accuracy tradeoff when the capability gap between the two models is smaller, we replace the original \texttt{gpt-4.1-nano} Primary
with the stronger but costlier \texttt{gpt-4.1-mini}.  
This boosts the stand-alone Primary accuracy from $0.56$ to
$0.77$—only $\sim\!6$ pp below the Guardian—and raises the token
price four-fold.  
Even in this compressed regime CA still delivers a meaningful
improvement over cost-matched random routing: at
$\alpha\!=\!0.05$ it gains $+2$ pp in accuracy while invoking the
Guardian on just one quarter of the queries, and at
$\alpha\!=\!0.025$ it \emph{matches} the Guardian’s accuracy for
$40\%$ of the cost.
The detailed numbers are collected in Table~\ref{tab:truqa_model_ablation},
and the corresponding cost–accuracy frontier is visualized in
Figure~\ref{fig:truqa_mini_ablation}.

\begin{table}[h]
\centering
\small
\caption{Model-ablation results on TruthfulQA with
\texttt{gpt-4.1-mini} as the Primary.  Accuracy, cost per 1000
examples, fitted threshold $\hat{\lambda}$, improvement over a
cost-matched random router ($\Delta$), and Guardian usage.  Means
$\pm$ one standard deviation across 30 trials.}
\label{tab:truqa_model_ablation}
\begin{tabular}{lccccc}
\toprule
\textbf{Policy} & \textbf{Accuracy} & \textbf{Cost (\$/1000)} &
$\boldsymbol{\hat{\lambda}}$ & $\boldsymbol{\Delta}$ &
\textbf{Guardian \%} \\
\midrule
Primary (\texttt{4.1-mini}) & $0.7738 \pm 0.0113$ & $0.126 \pm 0.000$ & -- & --      & $0.0\%$ \\[2pt]
CA ($\alpha=0.050$)         & $0.8156 \pm 0.0194$ & $0.265 \pm 0.032$ & $0.452 \pm 0.082$ & $+0.021$ & $23.9 \pm 5.0\%$ \\
CA ($\alpha=0.025$)         & $0.8345 \pm 0.0208$ & $0.375 \pm 0.064$ & $0.669 \pm 0.094$ & $+0.026$ & $41.2 \pm 10.7\%$ \\[2pt]
Guardian (\texttt{4.1})     & $0.8328 \pm 0.0088$ & $0.615 \pm 0.001$ & -- & --      & $100.0\%$ \\
\bottomrule
\end{tabular}
\end{table}

\begin{figure}[H]
  \centering
  \includegraphics[width=0.5\textwidth]{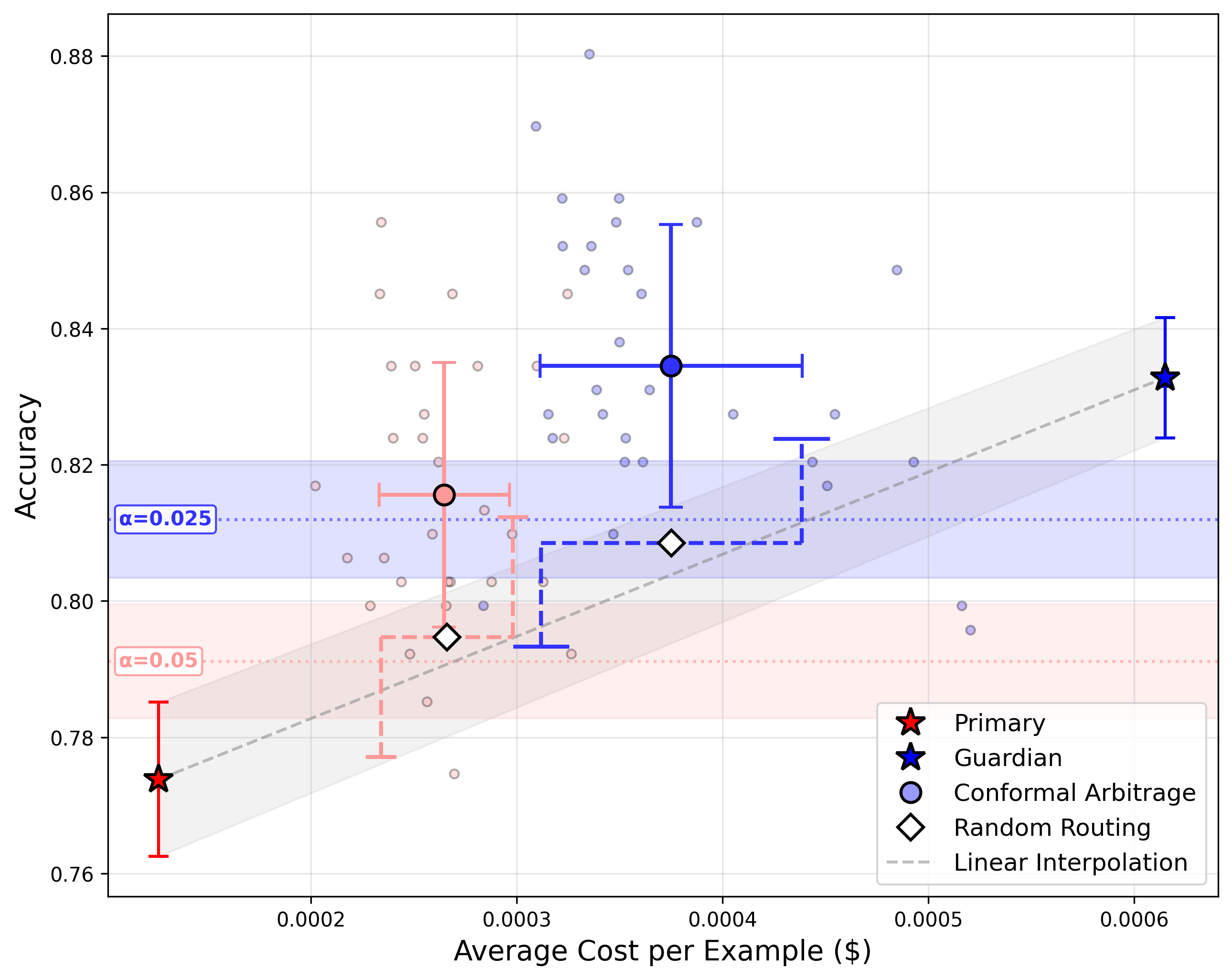}
  \caption{Cost–accuracy frontier on TruthfulQA when the Primary is
  \texttt{gpt-4.1-mini}.  Large markers denote means over 30 trials
  (error bars: $\pm1$ s.d.); faint dots show individual CA runs.  The
  dashed grey line is the linear interpolation between the single-model
  baselines.}
  \label{fig:truqa_mini_ablation}
\end{figure}

\newpage

\section{PKU-SafeRLHF: Additional Experimental Details}
\label{app:pku_details}

\subsection{Prompt for Score Elicitation}
\label{app:pku_prompt}

\noindent
\textbf{System message}
\begin{quote}\scriptsize\ttfamily
You are an expert assistant-evaluator.

\# Task\\
Score the \emph{helpfulness} of \textbf{each} assistant reply on a continuous scale from 0.00 to 1.00 (use two decimal places).  

-- Evaluate \textbf{only technical helpfulness}: factual correctness, completeness, clarity, and relevance to the user’s request.\\
-- \textbf{Ignore} safety, ethics, policy compliance, or any notion of harm.  

\# Reasoning\\
1. Analyse the user’s intent.\\
2. Independently assess each reply’s strengths and weaknesses.\\
3. Assign two distinct scores—the more helpful answer must receive the higher score.  

\# Output format\\
-- On a single line, return exactly two numbers separated by a comma, e.g.\ `0.37, 0.54`.\\
-- Scores must differ and lie in [0.00, 1.00].\\
-- \textbf{No explanations, no extra text}, just the two numbers.
\end{quote}

\vspace{1ex}
\noindent
\textbf{User message template (\textit{filled per instance})}
\begin{quote}\scriptsize\ttfamily
<USER PROMPT>\\
\{<verbatim user prompt text>\}\\
</USER PROMPT>\\[0.5ex]
<ASSISTANT REPLY 0>\\
\{<first candidate reply>\}\\
</ASSISTANT REPLY 0>\\[0.5ex]
<ASSISTANT REPLY 1>\\
\{<second candidate reply>\}\\
</ASSISTANT REPLY 1>\\[0.5ex]
Scores:
\end{quote}

 We parse the single-line response as two floats, enforce strict inequality by perturbing ties by $\pm0.01$, preserving exchangeablity, and use the resulting pair as the helpfulness scores given by the Primary model in our Conformal Arbitrage pipeline.

\subsection{Numerical Results}
\label{app:pku_numerical}

We provide the complete numerical results for the PKU-SafeRLHF
experiment introduced in Section~\ref{sec:experiments}.
Table~\ref{tab:pku_main} aggregates performance over
30 independent calibration/evaluation splits.
\textbf{Accuracy} is the fraction of prompts whose chosen answer
matches the annotator-preferred \emph{more-helpful} response, while
\textbf{Severity-loss} measures the average excess severity of the
selected answer above the safer one
(\(0\le \text{sev}\le 3\); lower is better).
As guaranteed by theory, every CA configuration respects the
finite-sample bound
\(\text{Severity-loss}\le\alpha\)
while tracing an efficient helpfulness–harmlessness frontier that
strictly dominates random routing.

\begin{table}[h]
\centering
\caption{PKU-SafeRLHF helpfulness–harmlessness trade-off. Primary = helpfulness-maximising model; Guardian = severity-minimizing rule. Mean $\pm$ std over 30 trials.}
\label{tab:pku_main}
\begin{tabular}{lccccc}
\toprule
\textbf{Policy} & \textbf{Accuracy} & \textbf{Severity-loss} & $\boldsymbol{\hat{\lambda}}$ & $\boldsymbol{\Delta}$ & \textbf{Guardian \%} \\
\midrule
Primary                       & $0.519 \pm 0.019$ & $0.676 \pm 0.033$ & --                & --       & $0.0\%$  \\[2pt]
CA ($\alpha=0.60$)            & $0.475 \pm 0.029$ & $0.571 \pm 0.070$ & $0.206 \pm 0.088$ & $+0.012$ & $19.0 \pm 9.4\%$ \\
CA ($\alpha=0.50$)            & $0.443 \pm 0.026$ & $0.482 \pm 0.053$ & $0.354 \pm 0.051$ & $+0.028$ & $35.6 \pm 5.3\%$ \\
CA ($\alpha=0.40$)            & $0.393 \pm 0.034$ & $0.379 \pm 0.064$ & $0.495 \pm 0.061$ & $+0.033$ & $51.8 \pm 8.0\%$ \\
CA ($\alpha=0.30$)            & $0.325 \pm 0.026$ & $0.245 \pm 0.043$ & $0.619 \pm 0.022$ & $+0.037$ & $71.7 \pm 4.9\%$ \\
CA ($\alpha=0.20$)            & $0.270 \pm 0.018$ & $0.161 \pm 0.021$ & $0.681 \pm 0.007$ & $+0.028$ & $82.2 \pm 2.1\%$ \\
CA ($\alpha=0.10$)            & $0.214 \pm 0.016$ & $0.080 \pm 0.022$ & $0.777 \pm 0.014$ & $+0.015$ & $91.8 \pm 1.9\%$ \\[2pt]
Guardian                      & $0.156 \pm 0.011$ & $0.000 \pm 0.000$ & --                & --       & $100.0\%$ \\
\bottomrule
\end{tabular}
\end{table}

Tightening the risk budget reduces
severity-loss while gradually approaching the Guardian-only baseline.
At \(\alpha=0.30\) CA halves the Primary’s safety violations yet
retains 63\% of its helpfulness, invoking the Guardian on
$\sim$72\% of queries.
Even under the strictest budget (\(\alpha=0.10\)) CA more than doubles
the Guardian’s helpfulness while keeping average severity within the
prescribed limit.

\section{MMLU}
\label{app:mmlu}

We next evaluate Conformal Arbitrage (CA) on the \emph{Massive Multitask Language Understanding} benchmark (\textsc{MMLU}; \citep{hendrycks2021measuring}).  Unless otherwise noted, the pipeline, models, prompts, cost accounting, and random–router baselines are identical to the TruthfulQA setup in Section \ref{sec:experiments}; below we list only the divergences that are specific to \textsc{MMLU}. Both models receive the same JSON-forced multiple-choice prompt used for TruthfulQA (Appendix \ref{app:truthfulqa_prompt_full}); we simply drop the TruthfulQA pre-amble and insert the \textsc{MMLU} question and four answer strings verbatim.

\paragraph{Dataset}
\textsc{MMLU} comprises almost $\sim\!16$k multiple choice questions across 57 subject areas covering high-school, undergraduate, and professional curricula.  We load the public \texttt{cais/mmlu} distribution via \texttt{datasets} and collapse the original \texttt{train}/\texttt{validation}/\texttt{test} splits into one pool. For each \emph{trial} we draw a fresh, balanced sample of
\(N_\text{tot}=1{,}000\) questions, allocating
\(n=500\) for calibration and the remaining \(500\) for evaluation.
Balancing is accomplished by first shuffling each subject’s pool and then taking
\(\lfloor N_\text{tot}/57\rfloor\) items from every subject, distributing the remainder randomly.

\paragraph{Results}
Although it is of less average gain compared to TruthfulQA, Conformal Arbitrage still traces an efficient frontier
that beats cost-matched random routing for most values of \(\alpha\) apart from the extremes. We can see that, in particular, the performance of CA degrades at the higher and lower values of $\alpha$ compared to the middle range. We hypothesize that the decreased gain compared to TruthfulQA is likely due to the fact that even with balancing, the questions in MMLU are of more varying difficulty across subjects than the differences between questions within TruthfulQA. 
Nevertheless, at \(\alpha=0.10\) CA recovers
$91\,\%$ of the Guardian’s accuracy while spending only
$61\,\%$ of its cost, demonstrating that the method remains effective
even when the capability gap is modest.

\begin{table}[h]
\centering
\caption{Accuracy, cost per 1000 examples, $\hat{\lambda}$, $\Delta$ above random baseline, and Guardian usage (mean ± std over 30 trials; calibration \(n=500\)).}
\label{tab:mmlu_restricted}
\begin{tabular}{lccccc}
\toprule
\textbf{Policy} & \textbf{Accuracy} & \textbf{Cost (\$/1000)} & $\boldsymbol{\hat{\lambda}}$ & $\boldsymbol{\Delta}$ & \textbf{Guardian \%} \\
\midrule
Primary           & $0.591 \pm 0.011$ & $0.035 \pm 0.000$ & --              & --       & $0.0\%$          \\
CA ($\alpha=0.25$) & $0.618 \pm 0.019$ & $0.111 \pm 0.034$ & $0.126 \pm 0.111$ & $-0.005$ & $13.0 \pm 5.6\%$ \\
CA ($\alpha=0.20$) & $0.663 \pm 0.021$ & $0.194 \pm 0.024$ & $0.423 \pm 0.059$ & $+0.011$ & $24.5 \pm 3.3\%$ \\
CA ($\alpha=0.15$) & $0.706 \pm 0.022$ & $0.317 \pm 0.057$ & $0.651 \pm 0.065$ & $+0.008$ & $42.9 \pm 9.5\%$ \\
CA ($\alpha=0.10$) & $0.753 \pm 0.020$ & $0.416 \pm 0.029$ & $0.771 \pm 0.021$ & $+0.018$ & $55.8 \pm 4.1\%$ \\
CA ($\alpha=0.05$) & $0.802 \pm 0.026$ & $0.624 \pm 0.065$ & $0.924 \pm 0.058$ & $-0.005$ & $86.9 \pm 9.8\%$ \\
Guardian          & $0.828 \pm 0.008$ & $0.676 \pm 0.004$ & --              & --       & $100.0\%$        \\
\bottomrule
\end{tabular}
\end{table}

\begin{figure}[h]
  \centering
  \includegraphics[width=0.90\textwidth]{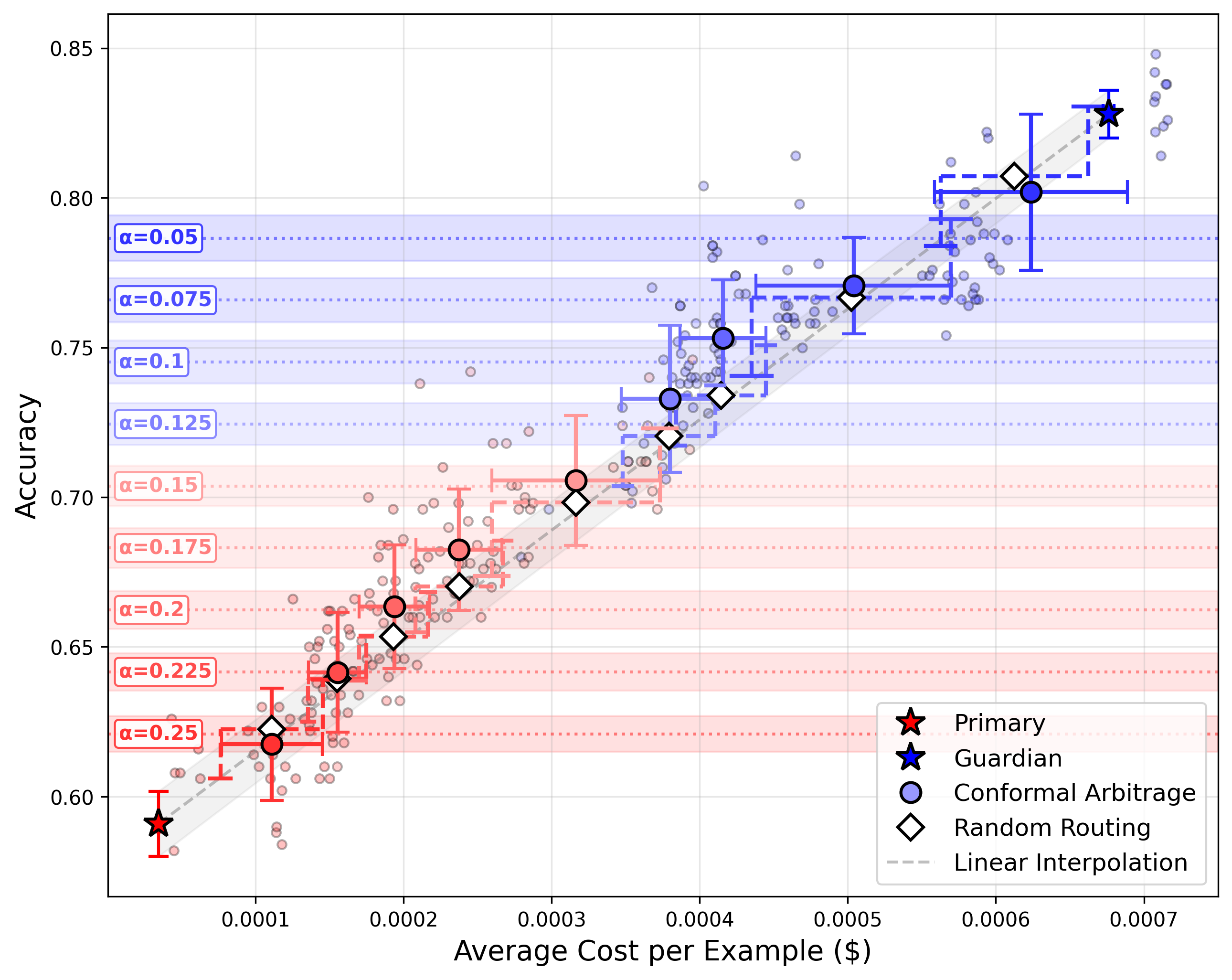}
  \caption{Cost–accuracy frontier on MMLU.  Mean $\pm$ std over 30 trials. Faint dots show individual CA runs.  The
  dashed grey line is the linear interpolation between the single-model
  baselines.}
  \label{fig:truqa_model_ablation_plot}
\end{figure}

\end{document}